\documentclass[11pt]{article}
\usepackage[utf8]{inputenc}
\usepackage[top=1in, left=1in, right=1in, bottom=1in]{geometry}

\usepackage[utf8]{inputenc} 
\usepackage[T1]{fontenc}    
\usepackage{booktabs}       
\usepackage{amsfonts}       
\usepackage{nicefrac}       
\usepackage{microtype}  
\usepackage{xcolor}         

\usepackage{etoc}
\usepackage[colorlinks,linkcolor=red,filecolor=blue,citecolor=blue,urlcolor=blue]{hyperref}
\usepackage{booktabs}       
\usepackage{amsfonts}       
\usepackage{amsmath}
\usepackage{graphicx}
\usepackage{amsthm}
\usepackage{amssymb}
\usepackage{multirow}
\usepackage{boldline}
\usepackage{makecell}
\usepackage[round]{natbib}
\usepackage{amssymb}%
\usepackage{mathtools}
\usepackage{stmaryrd}
\usepackage[T1]{fontenc}
\usepackage{xargs}
\usepackage{xcolor}
\usepackage{wrapfig}
\usepackage{bm,bbm}

\usepackage[utf8]{inputenc} 
\usepackage[T1]{fontenc}    
\usepackage{hyperref}       
\usepackage{url}            
\usepackage{booktabs}       
\usepackage{amsfonts}       
\usepackage{nicefrac}       
\usepackage{microtype}      
\usepackage{xcolor}  

\usepackage{times}  
\usepackage{helvet}  
\usepackage{courier}  
\usepackage{graphicx} 
\usepackage{natbib}  
\usepackage{caption} 
\usepackage{algorithm}
\usepackage{algorithmic}
\usepackage{wrapfig}

\usepackage{epsfig}

\usepackage{amsmath}
\usepackage{amssymb}
\usepackage{bm,amsmath,amsthm,amssymb,enumitem}
\usepackage{subfig}
\usepackage{xargs}
\usepackage{stmaryrd}
\usepackage{mdframed}
\usepackage{booktabs}

\newmdtheoremenv{theo}{Theorem}
\newmdtheoremenv{lemm}{Lemma}
\newmdtheoremenv{coro}{Corollary}

\newtheorem{assumption}{H\!\!}

\newtheorem*{Lemma*}{Lemma}
\newtheorem*{Theorem*}{Theorem}
\newtheorem*{Corollary*}{Corollary}

\newcommand{\eqsp}{\;}
\newcommand{\beq}{\begin{equation}}
\newcommand{\eeq}{\end{equation}}
\newcommand{\eqdef}{\mathrel{\mathop:}=}

\def\EE{\mathbb{E}}
\def\ie{\textit{i.e.,}}
\newcommand{\prop}[1]{\textsc{t}_{#1}}

\def\rset{\ensuremath{\mathbb{R}}}
\def\zset{\ensuremath{\mathcal{Z}}}
\def\xset{\ensuremath{\mathcal{X}}}
\def\bset{\ensuremath{\mathcal{B}}}
\def\accept{\ensuremath{\mathcal{A}_\theta}}
\def\compaccept{\ensuremath{\mathcal{A}^*_\theta}}
\def\thresh{\ensuremath{\textsf{th}}}

\def\algo{\textsc{STANLEY}}

\def\stepsize{\mathsf{\gamma}}

\usepackage{newfloat}
\usepackage{listings}
\floatstyle{ruled}
\newfloat{listing}{tb}{lst}{}
\floatname{listing}{Listing}

\begin{document}

\title{\bf STANLEY: Stochastic Gradient Anisotropic\\ Langevin Dynamics for Learning Energy-Based Models}

\author{\vspace{0.3in}\\\\
\textbf{Belhal Karimi}, \ \textbf{Jianwen Xie}, \  \textbf{Ping Li} \\\\
Cognitive Computing Lab\\
Baidu Research\\
10900 NE 8th St. Bellevue, WA 98004, USA\\\\
  \texttt{\{belhal.karimi,jianwen.kenny,pingli98\}@gmail.com}
}

\date{\vspace{0.5in}}
\maketitle

\begin{abstract}
 \noindent We propose in this paper, \textbf{\algo}, a \textbf{ST}ochastic gradient \textbf{AN}isotropic \textbf{L}ang\textbf{E}vin d\textbf{Y}namics, for sampling high dimensional data.
With the growing efficacy and potential of Energy-Based modeling, also known as non-normalized probabilistic modeling, for modeling a generative process of different natures of high dimensional data observations, we present an end-to-end learning algorithm for Energy-Based models (EBM) with the purpose of improving the quality of the resulting sampled data points.
While the unknown normalizing constant of EBMs makes the training procedure intractable, resorting to Markov Chain Monte Carlo (MCMC) is in general a viable option.
Realizing what MCMC entails for the EBM training, we propose in this paper, a novel high dimensional sampling method, based on an anisotropic stepsize and a gradient-informed covariance matrix, embedded into a discretized Langevin diffusion.
We motivate the necessity for an anisotropic update of the negative samples in the Markov Chain by the nonlinearity of the backbone of the EBM, here a Convolutional Neural Network.
Our resulting method, namely \algo, is an optimization algorithm for training Energy-Based models via our newly introduced MCMC method.
We provide a theoretical understanding of our sampling scheme by proving that the sampler leads to a geometrically uniformly ergodic Markov Chain.
Several image generation experiments are provided in our paper to show the effectiveness of our method.\\

\end{abstract}

\newpage

\section{Introduction}
The modeling of a data generating process is critical for many modern learning tasks.
A growing interest in generative models within the realm of computer vision has led to multiple interesting solutions.
In particular, Energy-based models (EBM)~\citep{zhu1998filters,lecun2006tutorial}, are a class of generative models that learns high dimensional and complex (in terms of landscape) representation/distribution of the input data.
EBMs have been used in several applications including computer vision~\citep{ngiam2011learning, xie2016theory,du2019implicit}, natural language processing~\citep{mikolov2013distributed,deng2020residual},  density estimation~\citep{wenliang2019learning,song2020sliced} and reinforcement learning~\citep{haarnoja2017reinforcement}. Formally, EBMs are built upon an unnormalized log probability, called the energy function, that is not required to sum to one as standard log probability functions.
This noticeable feature allows for more freedom in the way one parameterizes the EBM.
For instance, Convolutional Neural Network can be employed to parameterize this function, see~\citet{xie2016theory}.

The training procedure of such models consists of finding an energy function that assigns to lower energies to observations than unobserved points.
This phase can be cast as an optimization task and several ways are possible to solve it.
In this paper, we will focus on training the EBM via Maximum Likelihood Estimation (MLE). Particularly, while using MLE to fit the EBM on observed data, the high non-convexity of the loss function leads to a non closed form maximization step. 
In general, gradient based optimization methods are thus used during that phase.
Besides, given the intractability of the normalizing constant of our model, the aforementioned gradient, which is an intractable integral, needs to be approximated.
A popular and efficient way to conduct such approximation is to use Monte Carlo approximation where the samples are obtained via Markov Chain Monte Carlo (MCMC)~\citep{meyn2012markov}. The goal of this embedded MCMC procedure while training the EBM is to synthesize new examples of the input data and use them to approximate quantities~of~interest.

Hence, the sampling phase is crucial for both the EBM training speed and its final accuracy in generating new synthetic samples.
The computational burden of those MCMC transitions, at each iteration of the EBM training procedure, is alleviated via different techniques in the literature.
For instance, in~\citet{nijkamp2019learning}, the authors develop a short-run MCMC as a flow-based generator mechanism despite its non convergence property. Other principled approach, as in~\citet{hinton2002training}, keeps in memory the final chain state under the previous global model parameter and uses it as the initialization of the current chain.
The heuristic of such approach is that along the EBM iterations, the conditional distributions, depending on the model parameter, are more and more similar and thus using a good sample from the previous chain is in general a good sample of the current one.
Though, this method can be limited during the first iterations of the EBM training since when the model parameter changes drastically, the conditional distributions do change too, and samples from two different chains can be quite inconsistent.
Several extensions modifying the way the chain is initialized can be found in~\citet{welling2002new,gao2018learning,du2019implicit}.

An interesting line of work in the realm of MCMC-based EBM tackles the biases induced by stopping the MCMC runs too early. 
Indeed, it is known, see~\citet{meyn2012markov}, that before convergence, MCMC samples are biased and thus correcting this bias while keeping a short and less expensive run is an appealing option.
Several contributions aiming at removing this bias for improved MCMC training include coupling MCMC chains, see~\citet{qiu2020unbiased,jacob2020unbiased} or simply estimating this bias and correct the chain afterwards, see~\citet{du2021improved}.

Here, our work is in line with the context of -- high-dimensional data, -- EBM parameterized by deep neural networks and -- MLE-based optimization via MCMC, which make our method particularly attractive to all of the above combined. 
We also consider the case of a short-run MCMC for the training of an EBM.
Rather than focusing on debiasing the chain, we develop a new sampling scheme where the goal is to obtain better samples from the target distribution in fewer MCMC transitions.
We consider that the shape of the target distribution, which inspires our proposed method, is of utmost importance to obtain such negative samples. Our \textbf{contributions} are summarized~below:
\begin{enumerate}
\item We develop \algo, an Energy-Based model training method that embeds a newly proposed \emph{convergent} and \emph{efficient} MCMC sampling scheme, focusing on curvature informed metrics of the target distribution one wants to sample from. 
\item Based on an anisotropic stepsize, our method, which is an improvement of the Langevin Dynamics, achieves to obtain negative samples from the Energy-Based model data distribution and improves the overall optimization algorithm. 
\item We prove the geometric ergodicity uniformly on any compact set of our MCMC method assuming some regularity conditions on the target distribution and on the backbone model of the EBM. 
\item  We empirically assess the effectiveness of our method on several image generation tasks, both on synthetic and real datasets including the Oxford Flowers 102 dataset, CIFAR-10 and CelebA. We conclude the work with an Image inpainting experiment on a benchmark~dataset.
\end{enumerate}

\noindent\textbf{Roadmap}. \ Section~\ref{sec:mcmc} introduces important notations and related work.
Section~\ref{sec:main} develops the main algorithmic contribution of this paper, namely \algo.
Section~\ref{sec:theory} presents our main theoretical results focusing on the ergodicity of the proposed MCMC sampling method.
Section~\ref{sec:numericals} presents several image generation experiments on both synthetic and real datasets.
The complete proofs of our theoretical results can be found in the supplementary material.

\section{On MCMC-based Energy Based Models}\label{sec:mcmc}

Given a stream of input data noted $x \in \xset \subset \rset^p$, the EBM is a Gibbs distribution defined as follows:
\beq\label{eq:ebm}
p_{\theta}(x) = \frac{1}{Z(\theta)} \mathrm{exp}(f_{\theta}(x)) \eqsp,
\eeq
where $\theta \in \Theta \subset \rset^d$ denotes the global parameters vector of our model and $Z(\theta) \eqdef \int_{x} \mathrm{exp}(f_{\theta}(x)) \textrm{d}x$ is the normalizing constant (with respect to $x$).
In particular, $f_{\theta}(x): \xset \to \rset$, is the energy function (up to a sign) that can be parameterized by a Convolutional Neural Network for instance.
The natural way of fitting model~\eqref{eq:ebm} is to employ Maximum Likelihood Estimation (MLE) maximizing the marginal likelihood $p(\theta)$, \ie\ finding the vector $\theta^*$ such that for any $x \in \xset$, 
\beq\label{eq:mle}
 \theta^*  = \arg \max \limits_{\theta \in \Theta} \mathsf{L}(\theta) = \arg \max \limits_{\theta \in \Theta} \EE_{q(x)}[\log p_{\theta}(x)]\eqsp,
 \eeq
where $q(x)$ denotes the true distribution of the input data $x$. The optimization task~\eqref{eq:mle} is not tractable in closed form and requires an iterative procedure in order to be solved.
The standard algorithm used to train EBMs is Stochastic Gradient Descent (SGD), see~\citet{robbins1951stochastic,bottou2007large}.
SGD requires having access to the gradient of the objective function $\log p(\theta)$ which requires computing an intractable integral, due to the high nonlinearity of the generally utilized parameterized model $f_\theta(x)$.
Given the general form defined in~\eqref{eq:ebm}, we have that:
\beq\notag
\begin{split}
\nabla \mathsf{L}(\theta)
 =   \EE_{q(x)}[\nabla_\theta f_\theta(x)] - \EE_{p_{\theta}(x)}[\nabla_\theta f_\theta(x)]\eqsp,
\end{split}
\eeq
and a simple Monte Carlo approximation of $\nabla \log p(\theta)$ yields the following expression of the gradient
\beq\label{eq:mcapprox}
\nabla \mathsf{L}(\theta) \approx  \frac{1}{n} \sum_{i=1}^n \nabla_\theta f_\theta(x^{q}_i) - \frac{1}{M} \sum_{m=1}^M \nabla_\theta f_\theta(z_m) \eqsp,
\eeq
where $\{z_m\}_{m=1}^M$ are samples obtained from the EBM $p_{\theta}(x)$ and $\{x^{q}_i\}_{i=1}^n$ are drawn uniformly from the true data distribution $q(x)$.
While drawing samples from the data distribution is trivial, the challenge during the EBM training phase is to obtain samples from the EBM distribution $p_{\theta}(x)$ for any model parameter $\theta \in \Theta$.
This task is generally performed using MCMC methods.
State-of-the-art MCMC used in the EBM literature include Langevin Dynamics, see~\citet{grenander1994representations,roberts1998optimal,roberts1996exponential} and Hamiltonian Monte Carlo (HMC), see~\citet{neal2011mcmc}.
Those methods are detailed in the sequel and are important concepts of our contribution.

\vspace{0.2in}
\noindent\textbf{Energy Based Models.}
Energy based models are a class of generative models that leverage the power of Gibbs potential and high dimensional sampling techniques to produce high quality synthetic image samples. EBMs are powerful tools for generative modeling tasks, as a building block for a wide variety of tasks. 
The main purpose of EBMs is to learn an energy function~\eqref{eq:ebm} that assigns low energy to a stream of observation and high energy values to other inputs.
Learning of such models is done via MLE~\citep{xie2016theory,du2019implicit} or Score Matching~\citep{hyvarinen2005estimation} or Noise Constrastive Estimation~\citep{gao2020flow}.
In several general applications, authors leverage the power of EBMs for  develop an energy-based optimal policy where the parameters of that energy function are provided by the reward of the overall system.
Learning EBMs with alternative strategies include contrastive divergence (CD)~\citep{hinton2002training,tieleman2008training}, noise contrastive estimation (NCE)~\citep{gutmann2010noise,gao2020flow}, introspective neural networks (INN)~\citep{lazarow2017introspective,jin2017introspective,lee2018wasserstein}, cooperative networks (CoopNets)~\citep{xie2018cooperative,xie2020cooperative, xie2021learning,xie2022cooperative,xie2022tale,zhao2023coopinit}, f-divergence~\citep{yu2020training}, and triangle divergence~\citep{han2019divergence,han2020joint}.
Recently, EBMs parameterized by modern neural networks have drawn much attention from the computer vision and machine learning communities. 
Successful applications with EBMs include image generation~\citep{xie2016theory,gao2018learning,du2019implicit,zhao2021learning,zheng2021patchwise}, videos~\citep{XieCVPR17,xie2021learning_pami}, 3D volumetric shapes~\citep{xie2018learning,xie2022generative}, 
unordered point clouds~\citep{xie2021generative}, 
texts~\citep{deng2020residual}, molecules~\citep{ingraham2019learning,du2020energy}, as well as image-to-image translation~\citep{xie2022cooperative,xie2021cycleCoopNets}, unpaired cross-domain image translation~\citep{xie2021cycleCoopNets,song2023progressive}, 
out-of-distribution detection~\citep{liu2020energy}, inverse optimal control~\citep{xu2022energy}, deep regression~\citep{gustafsson2020energy}, salient object detection~\citep{zhang2022energy} and latent space modeling~\citep{pang2020learning,zhang2021learning,jing2023ebmprior}.
Yet, unlike VAE~\citep{kingma2014auto} or GAN~\citep{goodfellow2014generative} EBMs enjoy from a single structure requiring training (versus several networks) resulting in more stability.
The use of implicit sampling techniques, such as MCMC, as detailed in the sequel, allows more flexibility by trading off sample quality for computation time.
Overall, the \emph{implicit} property of the EBM, seen as an energy function, makes it a tool of choice as opposed to \emph{explicit} generators that are limited to some design choices, such as the choice of the prior distribution for VAEs or both neural networks design in GANs.

\vspace{0.2in}
\noindent\textbf{MCMC procedures.}
Whether for sampling from a posterior distribution~\citep{robert2010introducing, han2017alternating, xie2019learning, zhang2020learning, an2021learning,zhu2023likelihood, xie2023tale}, or in general intractable likelihoods scenario~\citep{doucet2000sequential}, various inference methods are available.
Approximate inference is a partial solution to the inference problem and include techniques such as Variational Inference (VI)~\citep{wainwright2008graphical,freitas2001variational} or Laplace Approximation~\citep{wolfinger1993laplace,rue2009approximate}. 
Those methods allow the simplification of the intractable quantities and result in the collection of good, yet approximate, samples.
As seen in~\eqref{eq:mcapprox}, training an EBM requires obtaining samples from the model itself.
Given the nonconvexity of the structural model $f_\theta(\cdot)$ with respect to the model parameter $\theta$, direct sampling is not an option.
Besides, in order to update the model parameter $\theta$, usually through gradient descent type of methods~\citep{bottou2007large},  exact samples from the EBM are needed in order to compute a good approximation of its (intractable) gradient, see~\eqref{eq:mcapprox}.
To do so, we generally have recourse to MCMC methods.
MCMC are a class of inference algorithms that provide a principled iterative approach to obtain samples from any intractable distribution.
While being exact, the samples generally represent a larger computation burden than methods such as VI.
Increasing the efficiency of MCMC methods, by obtaining exact samples, in other words constructing a chain that converges faster, in fewer transitions is thus of utmost importance in the context of optimizing EBMs.
Several attempts have been proposed for the standalone task of posterior sampling through the use of Langevin diffusion, see the Unadjusted Langevin in~\citet{brosse2019tamed}, the MALA algorithm in~\citet{roberts1998optimal,roberts1996exponential,durmus2017fast} or leveraging Hamiltonian Dynamics as in~\citet{girolami2011riemann}.
We propose in the next section, an improvement of the Langevin diffusion with the ultimate goal of speeding the EBM training procedure.
Our method includes this latter improvement in an end-to-end learning algorithms for Energy-Based models.

\section{Gradient Informed Langevin Diffusion}\label{sec:main}

We now introduce the main algorithmic contribution of our paper, namely \algo.
\algo\ is a learning algorithm for EBMs, comprised of a novel MCMC method for sampling samples from the intractable model~\eqref{eq:ebm}.
We provide theoretical guarantees of our scheme in Section~\ref{sec:theory}.

\subsection{Preliminaries on Langevin MCMC based EBM}
State-of-the-art MCMC sampling algorithm, particularly used during the training procedure of EBMs, is the discretized Langevin diffusion, cast as Stochastic Gradient Langevin Dynamics (SGLD), see~\citet{welling2011bayesian}.
In particular, several applications using EBM and SGLD have thrived in image generation, natural language processing or even biology~\citep{du2020energy}.
Yet, the choice of the proposal, generally Gaussian, is critical for improving the performances of both the sampling step (inner loop of the whole procedure) and the EBM training.
We recall the vanilla discretized Langevin diffusion used in the related literature as follows:
\beq\notag
z_k = z_{k-1} + \frac{\gamma}{2} \nabla \log \pi_\theta(z_k) + \sqrt{\gamma} B_k \eqsp,
\eeq
where $\pi_\theta(\cdot):= p(\cdot,\theta)$ is the target potential one needs samples from and defined in \eqref{eq:ebm}, $z_k$ represents the states of the chains at iteration $k$, \ie\ the generated samples in the context of EBM, $k$ is the MCMC iteration index and $B_k$ is the Brownian motion, usually set as a Gaussian noise and which can be written as $B_k \eqdef \epsilon \, \xi_k$ where $\xi_k$ is a standard Gaussian random variable and $\epsilon$ is a scaling factor for implementation purposes.
This method directs the proposed moves towards areas of high probability of the stationary distribution $\pi_\theta$, for any $\theta \in \Theta$ , using the gradient of $\log \pi_{\theta}$ and has been the object of several studies~\citep{girolami2011riemann,cotter2013mcmc}.
In high dimensional and highly nonlinear settings, the burden of computing this gradient for a certain number of MCMC transitions leads to a natural focus: improving of the sampling scheme by assimilating information about the landscape of the target distribution while keeping its ease of implementation.

\subsection{\algo, an Anisotropic Energy Based Modeling Approach}

Given the drawbacks of current MCMC methods used for training EBMs, we introduce a new sampler based on the Langevin updates presented above in Step~\ref{line:langevin} of Algorithm~\ref{alg:anila}.

\begin{algorithm}[t]
\caption{\algo\ for Energy-Based model} \label{alg:anila}
\begin{algorithmic}[1]
\STATE \textbf{Input}: Number of iterations $T$, number of MCMC steps $K$ and of samples $M$, learning rates $\{\eta_t\}_{t >0}$, stepsize threshold $\thresh$, initial value $\theta_0$, MCMC initialization $\{ z_{0}^m \}_{m=1}^M$ and data~$\{ x_{i} \}_{i=1}^n$.
\FOR{$t=1$ to $T$}
\STATE Compute the anisotropic stepsize as follows: \label{line:step}
\beq\label{eq:step}
\stepsize_t = \thresh /\max(\thresh, | \nabla f_{\theta_t}(z_{t-1}^m) |) \eqsp.
\eeq
\STATE Draw $M$ samples $\{ z_{t}^m \}_{m=1}^M$ from the objective potential~\eqref{eq:ebm} via Langevin diffusion:\label{line:langevin}
\FOR{$k=1$ to $K$}
\STATE Construct the Markov Chain with the Brownian motion $\mathsf{B}_k$ as follows:
\beq\label{eq:anila}
z_{k}^{m} = z_{k-1}^m + \stepsize_t  \nabla f_{\theta_t}(z_{k-1}^m) /2 + \sqrt{\stepsize_t} \mathsf{B}_k \eqsp,
\eeq
\ENDFOR
\STATE Assign $\{ z_{t}^m \}_{m=1}^M \leftarrow \{ z_{K}^m \}_{m=1}^M$.
\STATE Sample $m$ positive observations $\{ x_{i} \}_{i=1}^m$ from the empirical data distribution.
\STATE Compute the gradient of the empirical log-EBM~\eqref{eq:ebm}:
\beq\notag
\begin{split}
\nabla \mathsf{L}(\theta_t)
 =& \mathbb{E}_{p_{\text {data }}}\left[\nabla_{\theta} f_{\theta_t}(x)\right]-\mathbb{E}_{p_{\theta}(x)}\left[\nabla_{\theta_t} f_{\theta}(z_t)\right] \approx  \frac{1}{n} \sum_{i=1}^{n} \nabla_{\theta} f_{\theta_t}\left(x_{i}\right)-\frac{1}{M} \sum_{m=1}^{M} \nabla_{\theta} f_{\theta_t}\left(z_K^m\right)\eqsp.
\end{split}
\eeq
\STATE Update the vector of global parameters of the EBM:\label{line:gradient}
$\theta_{t+1} = \theta_{t} + \eta_t \nabla \mathsf{L}(\theta_t)$ 
\ENDFOR
\STATE \textbf{Output:} Vector of fitted parameters $\theta_{T+1}$.
\end{algorithmic}
\end{algorithm}

\vspace{0.1in}\noindent
\textbf{Intuitions behind the efficacy of \algo:}
Some past modifications have been proposed in particular to optimize the covariance matrix of the proposal of the general MCMC procedure in order to better stride the support of the target distribution. 
Langevin Dynamics is one example of those improvements where the proposal is a Gaussian distribution where the mean depends on the gradient of the log target distribution and the covariance depends on some Brownian motion.
For instance, in~\citet{atchade2006adaptive,marshall2012adaptive}, the authors propose adaptive and geometrically ergodic Langevin chains. 
Yet, one important characteristic of our EBM problem, is that for each model parameter updated through the training iterations, the target distribution moves and the proposal should take that adjustment into account.
The techniques in~\citet{atchade2006adaptive,marshall2012adaptive} does not take the whole advantage of changing the proposal using the target distribution. 
In particular, the covariance matrix of the proposal is given by a stochastic approximation of the empirical covariance matrix. 
This choice seems completely relevant as soon as the convergence towards the stationary distribution is reached, in other words it would make sense towards the end of the EBM training, as the target distributions from a model parameter to the next one are similar. 
However, it does not provide a good guess of the variability during the first iterations since it is still very dependent on the initialization. 

Moreover, in ~\citet{girolami2011riemann}, the authors consider the approximation of a constant. Even though this simplification leads to ease of implementation, the curvature metric chosen by the authors need to be inverted, step that can be a computational burden if not intractable. 
Especially in the case we are considering in our paper, \ie\ ConvNet-based EBM, where the high nonlinearity would lead to intractable expectations. Therefore, in~\eqref{eq:step} and~\eqref{eq:anila} of Algorithm~\ref{alg:anila}, we propose a variant of Langevin Dynamics, in order to sample from a target distribution, using a full anisotropic covariance matrix based on the anisotropy and correlations of the target distribution, see the $\sqrt{\stepsize_t} \mathsf{B}_k$~erm.

\newpage

\section{Geometric Ergodicity of \algo\ }\label{sec:theory}
We will present our theoretical analysis for the Markov Chain constructed using Line 3-4 of Algorithm~\ref{alg:anila}. Let $\Theta$ be a subset of $\rset^d$ for some integer $d >0$.
We denote by $\zset$ the measurable space of $\rset^\ell$ for some integer $\ell >0$.
We define a family of stationary distribution $\left(\pi_\theta(z) \right)_{\theta \in \Theta}$, probability density functions with respect to the Lebesgue measure on the measurable space $\zset$. 
 This family of p.d.f. defines the stationary distributions of our newly introduced sampler.

\subsection{Notations and Assumptions}
For any chain state $z \in \zset$, we denote by $\Pi_\theta(z,\cdot)$ the transition kernel as defined in the \algo\ update in Line 4 of Algorithm~\ref{alg:anila}.
The objective of this section is to rigorously show that each transition kernel $\pi_\theta$, for any parameter $\theta \in \Theta$ is geometrically ergodic and that this result holds for any compact subset $\mathcal{C} \in \zset$.
As a background note, a Markov chain, as constructed Line 4, is said to be geometrically ergodic when $k$ iterations of the same transition kernel is converging to the stationary distribution of the chain with a geometric dependence on $k$. 

As in~\citet{allassonniere2015convergent}, we state the assumptions required for our analysis.
The first one is related to the continuity of the gradient of the log posterior and the unit vectors pointing in the direction of the sample $z$ and in the direction of the gradient of the log posterior distribution at $z$:
\begin{assumption}\label{ass:bounded}
For all $\theta \in \Theta$, the structural model $f_\theta(\cdot)$ satisfies:
\begin{equation}\notag
\begin{split}
\lim \limits_{|z| \to \infty} \frac{z}{|z|} \nabla f_{\theta}(z)  = - \infty \eqsp, \; \;
 \lim \sup \limits_{|z| \to \infty} \frac{z}{|z|} \frac{\nabla f_{\theta}(z) }{|\nabla f_{\theta}(z) |} < 0 \eqsp.
\end{split}
\end{equation}
\end{assumption}

Besides, we assume some regularity conditions of the stationary distributions with respect to~$\theta$:
\begin{assumption}\label{ass:contlogpi}
$\theta \to \pi_\theta$ and $\theta \to \nabla \log \pi_\theta$ are continuous on $\Theta$.
\end{assumption}

For a positive and finite function noted $V: \zset \mapsto \rset$, we define the V-norm distance between two arbitrary transition kernels $\Pi_1$ and $\Pi_2$ as follows:
\beq\notag
\| \Pi_1 - \Pi_2 \|_V \eqdef \sup \limits_{z \in \zset} \frac{\| \Pi_1(z, \cdot) - \Pi_2(z, \cdot) \|_V }{V(z)} \eqsp.
\eeq

The definition of this norm allows us to establish a convergence rate for our sampling method by deriving an upper bound of $\| \Pi_\theta^k - \pi_\theta \|_V$ where $k >0$ denotes the number of MCMC transitions.
We recall that $\Pi_\theta$ is the transition kernel defined by Line 4 of Algorithm~\ref{alg:anila} and $\pi_\theta$ is the stationary distribution of our Markov chain at a given EBM model $\theta$.
This quantity characterizes how close to the target distribution our chain is getting, after a finite time of iterations and will eventually formalize the \emph{V-uniform ergodicity} of our method.
We specify that strictly speaking, $\pi_\theta$ is a probability measure, and not a transition kernel. 
However $\| \Pi_\theta^k - \pi_\theta \|_V$ is well-defined if we consider $\pi_\theta$ as~a~kernel:
\beq\notag
\pi(z, \mathcal{C}) \eqdef \pi(\mathcal{C}) \quad \textrm{for} \quad \mathcal{C} \subset \zset, \quad z \in \zset \eqsp.
\eeq
Here, for some $\beta \in ] 0,1[$ we define the $V_\theta$ function, also know as the \emph{drift}, for all $z \in \zset$ as follows: 
\beq\label{eq:driftfunction}
V_\theta(z) \eqdef c_\theta \pi_\theta(z)^{-\beta} \eqsp,
\eeq
where $c_\theta$ is a constant, with respect to the chain state $z $, such that for all $z \in \zset$, $V_\theta(z) \geq 1$.
Note that the V norm depends on the chain state noted $z$ \emph{and} of the global model parameter $\theta$ varying through the optimization procedure.
Yet, in both main results, the ergodicity and the convergence rate, including the underlying drift condition, are established uniformly on the parameter space $\Theta$.
We also define the auxiliary functions, independent of the parameter~$\theta$~as:
\beq\label{eq:vfunctions}
V_1(z)  \eqdef \inf \limits_{\theta \in \Theta} V_\theta(z) \quad \textrm{and} \quad V_2(z)  \eqdef \sup \limits_{\theta \in \Theta} V_\theta(z) \eqsp,
\eeq
and assume the following:
\begin{assumption}\label{ass:V2}
There exists a constant $a_0 > 0$ such that for all $\theta \in \Theta $ and $z \in \zset$, the function $V_2^{a_0}(z)$, defined in~\eqref{eq:vfunctions}, is integrable against the kernel $\Pi_\theta(z, \cdot)$ and we have
\beq\notag
 \lim \sup  \limits_{a \to 0}  \sup \limits_{\theta \in \Theta, z \in \zset} \Pi_\theta V_2^a(z) = 1 \eqsp.
\eeq
\end{assumption}

\vspace{-0.2in}

\subsection{Convergence Results}
The result consists in showing V-uniform ergodicity of the chain, the irreducibility of the transition kernels and their aperiodicity, following~\citet{meyn2012markov,allassonniere2015convergent}. 
We also prove a drift condition which states that the transition kernels tend to bring back elements into a small set.
Then, V-uniform ergodicity of the transition kernels $(\Pi_\theta)_{\theta \in \Theta}$ boils down from the latter proven drift condition.

\vspace{0.1in}\noindent
\textbf{Important Note:} The stationary distributions depend on $\theta \in \Theta$ as they vary at each model update during the EBM optimization phase.
Thus uniform convergence of the chain is important in order to characterize the sampling phase \emph{throughout the entire training phase}.
Particularly at the beginning, the shape of the distributions one needs to sample from varies a lot from a parameter to another.

Theorem~\ref{thm:thm1} shows two important convergence results for our sampling method. 
First, it establishes the existence of a small set $\mathcal{O}$ leading to the crucially needed aperiodicity of the chain and ensuring that each transition moves towards a better state.
Then, it provides a uniform ergodicity result of our sampling method in \algo, via the so-called \emph{drift condition} providing the guarantee that our transition kernels $(\Pi_\theta)_{\theta \in \Theta}$ attract the states into the small set $\mathcal{O}$.
Moreover, the independence on the EBM model parameter $\theta$ of $V$ in~\eqref{thm:main2} leads to \emph{uniform} ergodicity as shown in the ~Corollary~\ref{coro:coro1}.

\begin{theo}\label{thm:thm1}
Assume H\ref{ass:bounded}-H\ref{ass:V2}.
For any $\theta \in \Theta$, there exists a drift function $V_\theta$, a set $\mathcal{O} \subset \zset$, a constant $0 < \epsilon \leq 1$ such that 
\beq\label{thm:main1}
\Pi_\theta(z, \bset) \geq  \epsilon \int_{\bset} \mathsf{1}_\xset(z)  \textrm{d}y \eqsp.
\eeq
Moreover there exists $0 < \mu < 1$, $\delta > 0$ and a drift function $V$, independent of $\theta$ such that for all $z \in \zset$:
\beq\label{thm:main2}
\Pi_\theta V(z) \leq \mu V(z) + \delta \mathsf{1}_{\mathcal{O}}(z) \eqsp.
\eeq
\end{theo}

\vspace{-0.1in}

\begin{coro}\label{coro:coro1}
Assume H\ref{ass:bounded}-H\ref{ass:V2}.
A direct consequence of Theorem~\ref{thm:thm1} is that the family of transition kernels $(\Pi_\theta)_{\theta \in \Theta}$ are uniformly ergodic,\ie\ for any compact $\mathcal{C} \subset \zset$, there exist constants $\rho \in ]0,1[$ and $e >0$ such for any MCMC iteration $k > 0$,we have:
\beq\label{coro:main}
\sup \limits_{z \in \mathcal{C}} \| \Pi_\theta^k u(\cdot) - \pi_\theta u(\cdot) \|_{V} \leq e \rho^k \| u \|_{V} \eqsp,
\eeq
where $V$ is the drift function  in Theorem~\ref{thm:thm1} and $u(\cdot)$ is any bounded function we apply a transition to.
\end{coro}
While Theorem~\ref{thm:thm1} is critical for proving the aperiodicity and irreducibility of the chain, we establish the geometric convergence speed of the chain. We do not only show the importance of the \emph{uniform} ergodicity of the chain, which makes it appealing for the EBM training since the model parameter $\theta$ is often updated, but we also derive a geometrical rate in Corollary~\ref{coro:coro1}.

We encourage the readers to read through the sketch of the main Theorem of our paper provided on the first page of the supplemental as we give the important details leading to the desired ergodicity results.
Those various techniques are common in the MCMC literature and we refer the readers to several MCMC handbooks such as~\citet{neal2011mcmc,meyn2012markov} for more understanding.

\section{Numerical Experiments}\label{sec:numericals}

We conduct a collection of experiments to show the effectiveness of our method, both on synthetic and real datasets.
After verifying the advantage of \algo\ on a Gaussian Mixture Model (GMM) retrieving the synthetic data observations, we then investigate its performance when learning a distribution over high-dimensional natural images such as pictures of flowers, see the Flowers dataset in~\citet{nilsback2008automated}, or general concepts featured in CIFAR-10~\citep{krizhevsky2009learning}.
For both methods, we use the Frechet Inception Distance (FID), as a reliable performance metrics as detailed in~\citet{heusel2017gans}.
In the sequel, we tune the learning rates over a fine grid and report the best result for all methods.
For our method \algo, the threshold parameter $\thresh$, crucial for the implementation of the stepsize~\eqref{eq:step} is tuned over a grid search as well.
As mentioned above, we also define a Brownian motion as $B_k \eqdef \epsilon \, \xi_k$, and tune the scaling factor $\epsilon$ for~better~performances.

\subsection{Toy Example: Gaussian Mixture Model}

\textbf{Datasets.}
We first demonstrate the outcomes of both methods including our newly proposed \algo\ for low-dimensional toy distributions.
We generate synthetic 2D rings data and use an EBM to learn the true data distribution and put it to the test of generating new synthetic samples.

\vspace{0.1in}\noindent
\textbf{Methods and Settings.}
We consider two methods. 
Methods are ran with \emph{nonconvergent} MCMC, \ie, we do not necessitate the convergence to the stationary distribution of the Markov chains.
The number of transitions of the MCMC is set to $K= 100$ per EBM iteration. 
We use a standard deviation of $0.15$ as in~\citet{nijkamp2020anatomy}.
Both methods have a constant learning rate of $0.14$.
The value of the threshold $\thresh$ for our \algo\ method is set to $\thresh = 0.01$.
The total number of EBM iterations is set to $T = 10\,000$.
The global learning rate $\eta$ is set to a constant equal to $0.0001$.

\vspace{0.1in}\noindent
\textbf{Network architectures.} 
For the backbone of the EBM model, noted $f_\theta(\cdot)$ in~\eqref{eq:ebm}, we chose a CNN of $5$ 2D convolutional layers and Leaky ReLU activation functions, with the leakage parameter set to $0.05$.
The number of hidden neurons varies between $32$ and $64$.

\newpage

\begin{figure}[h]
\centering
\includegraphics[width=6in]{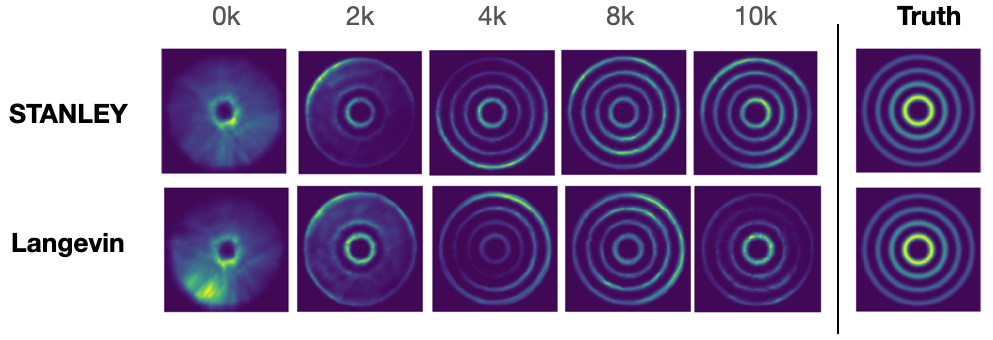}
  \caption{(Rings Toy Dataset) Top: our method, namely \algo\ Bottom: vanilla Langevin Dynamics. 
  Methods are used with the same backbone architecture. Generated samples are plotted through the iterations ever $2\,000$ steps.}
\label{fig:resultstoy}
\end{figure}

\vspace{0.1in}\noindent
\textbf{Results.}  
We observe Figure~\ref{fig:resultstoy} the outputs using both methods on the toy dataset.
While they achieve a great representation of the truth after a large number of iterations, we notice that \algo\ learns an energy that closely approximates the true density during the first thousands of iterations if the training process.
The sharpness of the data generated by \algo\ in the first iterations shows an empirically better ability to sample from the 2D dataset.

\subsection{Image Generation}

\textbf{Datasets.}
We run our method and several baselines detailed below on the \textit{CIFAR-10} dataset~\citep{krizhevsky2009learning} and the \emph{Oxford Flowers 102} dataset~\citep{nilsback2008automated}.
\textit{CIFAR-10}  is a popular computer-vision dataset of $50\,000$ training images and $10\,000$ test images, of size $32\times 32$. 
It is composed of tiny natural images representing a wide variety of objects and scenes, making the task of self supervision supposedly harder.
The \emph{Oxford Flowers 102} dataset is composed of 102 flower categories.
Per request of the authors, the images have large scale, pose and light variations making the task of generating new samples particularly challenging.

\begin{figure}[h]
\begin{center}
        \mbox{
        \includegraphics[width=3in]{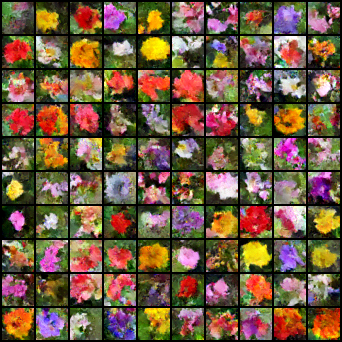}\hspace{0.2in}
        \includegraphics[width=3in]{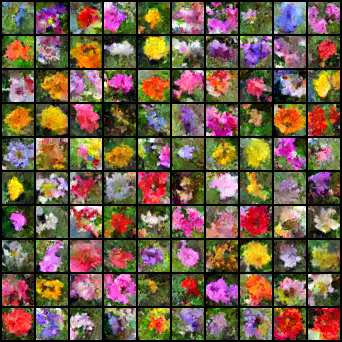}
        }
\end{center}
	\caption{Left: Langevin Method. Right: \algo\ . After 100k iterations.}
	\label{fig:flowers}
\end{figure}

\vspace{0.1in}\noindent
\textbf{Methods and Settings for the Flowers dataset.}
Nonconvergent MCMC are also used in this experiment and the number of MCMC transitions is set to $K = 50$.
Global learning parameters of the gradient descent update is set to $0.001$ for both methods.
We run each method during $T = 100\,000$ iterations and plot the results using the final vector of fitted parameters.

\vspace{0.1in}\noindent
\textbf{Methods and Settings for CIFAR-10.}
We employ the same nonconvergent MCMC strategies for this experiment.
The value of the threshold $\thresh$ for our \algo\ method is set to $\thresh = 0.0002$.
The total number of EBM iterations is set to $T = 100\,000$.
The global learning rate $\eta$ is set to a constant equal to $0.0001$.
In this experiment, we slightly change the last step of our method in Algorithm~\ref{alg:anila}.
Indeed, Line 11 in Algorithm~\ref{alg:anila} is not a plain Stochastic Gradient Descent here but we rather use the \textsc{Adam} optimizer~\citep{kingma2015adam}.
The scaling factor of the Brownian motion is~$0.01$.

\vspace{0.1in}\noindent
\textbf{Network architectures for both.} 
The backbone of the energy function for this experiment is a vanilla ConvNet composed of $3 \times 3$ convolution layers with stride $1$.
$5$ Convolutional Layers using ReLU activation functions are stacked.

\newpage

\vspace{0.1in}\noindent
\textbf{Results.} 
\textit{(Flowers)} \hspace{0.05in} Visual results are provided in Figure~\ref{fig:flowers} where we have used both methods to generate synthetic images of flowers.
For each threshold iterations number ($5\,000$ iterations) we sample $10\,000$ synthetic images from the EBM model under the current vector of parameters and use the same number of data observations to compute the FID similarity score as advocated in~\citet{heusel2017gans}.
The evolution of the FID values are reported in Figure~\ref{fig:fidall} (Left) through the iterations.
We note that our method outperforms the other baselines for all iterations threshold, including the Vanilla Langevin (in blue) which is an ablated form our \algo\ (no adaptive stepsize).

\begin{figure}[h]
\begin{center}
\mbox{
        \includegraphics[width=.5\textwidth]{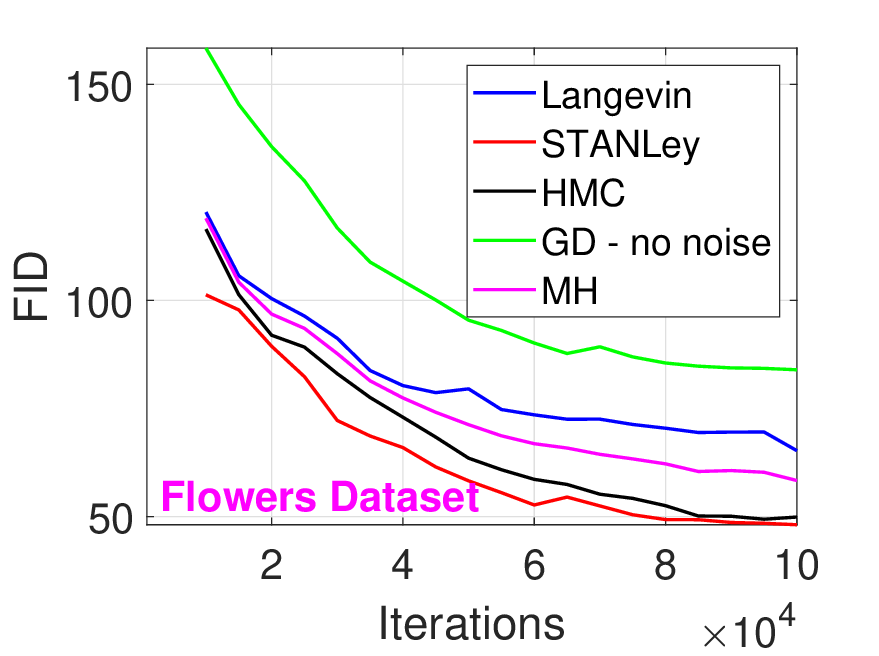} \hspace{-0.2in}
        \includegraphics[width=.5\textwidth]{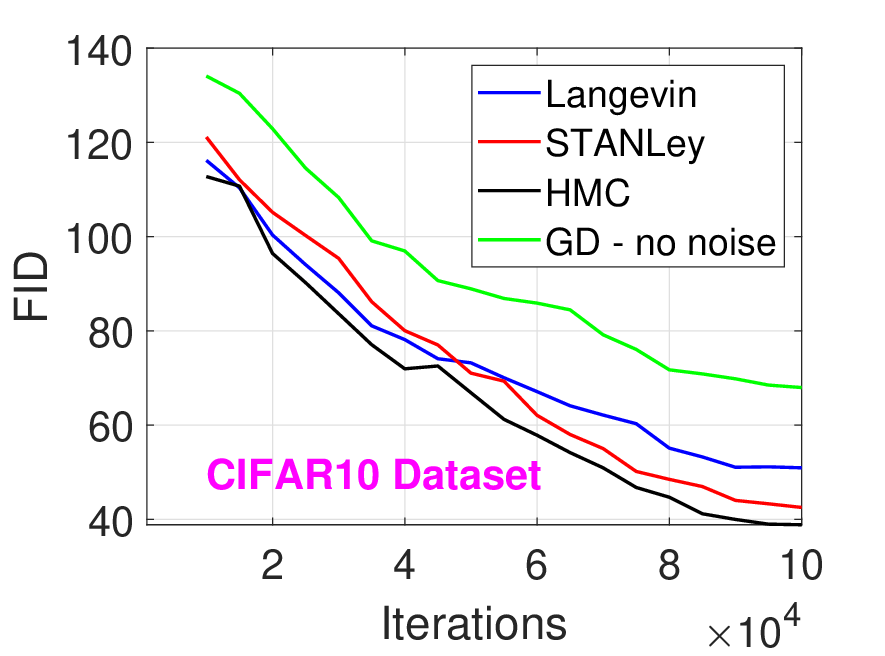}
}
\end{center}
	\caption{(FID values per method against 100k iterations elapsed). Left: Oxford Flowers dataset. Right: CIFAR-10.}
	\label{fig:fidall}
\end{figure}

\newpage

\textit{(CIFAR-10)} \hspace{0.05in} Visual results are provided in Figure~\ref{fig:cifar} where we have used both methods to generate synthetic images of flowers.
The FID values are reported in Figure~\ref{fig:fidall} (Right) and have been computed using $10\,000$ synthetic images from each model.
The similarity score is then evaluated every $5\,000$ iterations. 
While the FID curves for the Flowers dataset exhibits a superior performance of our method throughout the training procedure, we notice that in the case of CIFAR-10, vanilla method seems to be slightly better than \algo\. during the first iterations, \ie\ when the model is still learning the representation of the images.
Yet, after a certain number of iterations, we observe that \algo\ leads to more accurate synthetic images.
\emph{This behavior can be explained by the importance of incorporating curvature informed metrics into the training process when the parameter reaches a neighborhood of the solution}.

\begin{figure}[t]

    \mbox{\hspace{-0.2in}
        \includegraphics[width=.21\textwidth]{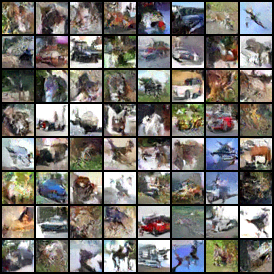}
        \includegraphics[width=.21\textwidth]{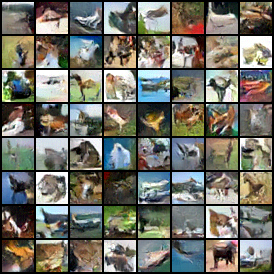}
	\includegraphics[width=.21\textwidth]{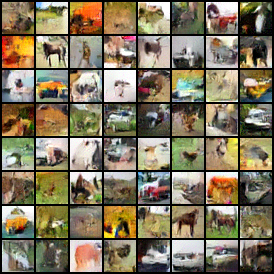}        
        \includegraphics[width=.21\textwidth]{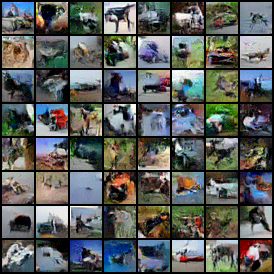}
        \includegraphics[width=.21\textwidth]{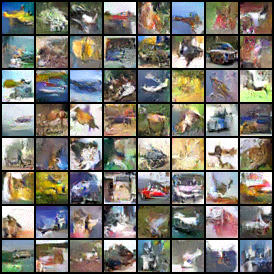}
        }
    \caption{1: Langevin 2: \algo\ 3: MH 4: HMC 5: GD without noise. After 100k iterations.}
	\label{fig:cifar}
\end{figure}

\subsection{Image Inpainting}

The image inpainting experiment aims to fill missing regions of a damaged image with synthesized content.

 \begin{figure}[b!]
 
\vspace{-0.2in}

 \centering
   \mbox{
        \includegraphics[width=3.2in]{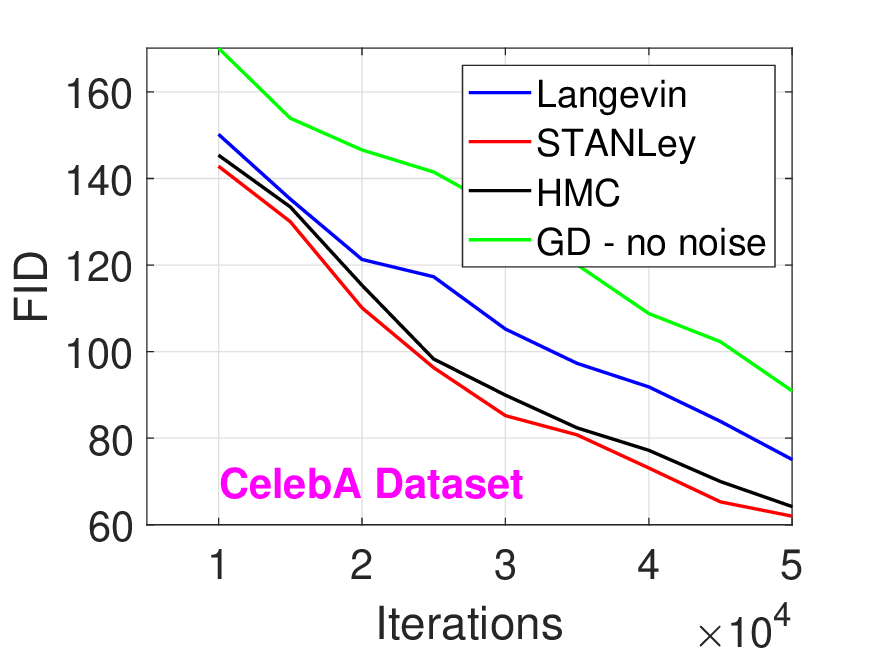}
        
    }
	\caption{FID values per method against 50k iterations elapsed.}
	\label{fig:fidceleb}\vspace{-0.15in}
\end{figure}

\vspace{0.1in}\noindent
\textbf{Datasets.}
We use the CelebA dataset~\citep{liu2015deep} to evaluate our learning algorithm, which contains more than 200k RGB color facial image. We use 100k images for training and 100 images for testing.

\vspace{0.1in}\noindent
\textbf{Methods and Settings.}
Nonconvergent MCMC are also used in this experiment and the number of MCMC transitions is set to $K = 50$.
Global learning parameters of the gradient descent update is set to $0.01$.
We run each method during $T = 50\,000$ iterations and plot the results using the final vector of fitted parameters.

\vspace{0.1in}\noindent
\textbf{Results.} 
Figure~\ref{fig:fidceleb} displays the FID curves for all methods.
We note that along the iterations, the \algo\ outperforms the other baseline and is similar to HMC, while only requiring first order information for the computation of the stepsize whereas HMC computes second order quantity.
 Even with second order information, the HMC samples does not lead to a better FID.

\begin{figure}[h]
\centering
\subfloat{\includegraphics[width=5in]{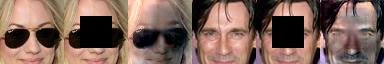}} \\
   \subfloat{\includegraphics[width=5in]{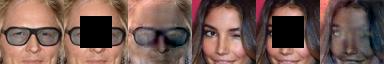}}
   \caption{Image inpainting. \textbf{Top:} \algo\ \textbf{Bottom:} Vanilla Langevin.}
   \label{fig:celebinpainting}
 \end{figure}
 
Figure~\ref{fig:celebinpainting} shows the visual check on different samples between our method and its ablated form, \ie\ the vanilla Langevin sampler based EBM. 
 In addition to a measurable metrics comparison, visual check provide empirical insights on the effectiveness of adding a curvature-informed stepsize in the sampler of our generative model as in \algo .

\section{Conclusion}\label{sec:conclusion}

We propose in this paper, an improvement of the so-called MCMC based Energy-Based models.
In the particular case of a highly nonlinear structural model of the EBM, more precisely a Convolutional Neural Network in our paper, we tackle the complex task of sampling negative samples from the energy function.
The multi-modal and highly curved landscape one must sample from inspire our technique called \algo, and based on a Stochastic Gradient Anisotropic Langevin Dynamics, that updates the Markov Chain using an anisotropic stepsize in the vanilla Langevin update.
We provide strong theoretical guarantees for our novel method, including uniform ergodicity and geometric convergence rate of the transition kernels to the stationary distribution of the chain.
Our method is tested on several benchmarks data and image generation tasks including toy and real datasets. 

\bibliographystyle{plainnat}
\bibliography{refs_scholar}

\newpage

\appendix
\section{Appendix: Proofs of the Theoretical Results}\label{app:proofs}

\subsection{Sketch of the Proof of Theorem \ref{thm:thm1}}

\textbf{Notations for the proof:}
We denote by $z \to \prop{\theta}(z',z)$, the pdf of the Gaussian proposal of Line 3 for any current state of the chain $z' \in \zset$ and dependent on the EBM model parameter $\theta$.
The transition kernel from $z$ to $z'$ is denoted by $\Pi_\theta(z, z')$.
$\zset$ is a subset of $\rset^\ell$ and $\bset$ is a Borel set of $\rset^\ell$.

The proof of our results are divided into two main parts.
We first prove the existence of a small set for our transition kernel $\Pi_\theta$, noted $\mathcal{O}$ showing that for any state, the Markov Chain moves away from it.
It constitutes the first step toward proving its irreducibility and aperiodicity.
Then, we will establish the so-called \emph{drift condition}, also known as the Foster-Lyapunov condition, crucial to proving the convergence of the chain.
The drift condition ensures the recurrence of the chain as the property that a chain returns to its initial state within finite time, see~\citet{roberts1998optimal,roberts2004general} for more details. 
Uniform ergodicity is then established as a consequence of those drift conditions and thus proving~\eqref{thm:main2}.

\medskip
\noindent \textbf{(i) Existence of a small set: }
Let $\mathcal{O}$ be a compact subset of the state space $\zset$. We recall the definition of the transition kernel in the case of a Metropolis adjustment and for any model parameter $\theta \in \Theta$ and state $z \in \zset$:
\beq\notag
\begin{split}
\Pi_\theta(z, \bset) = & \int_{\bset} \alpha_\theta(z, y) \prop{\theta}(z,y) \textrm{d}y 
 + \mathsf{1}_{\bset(z)}\int_{\zset} (1 - \alpha_\theta(z, y)) \prop{\theta}(z,y) \textrm{d}y \eqsp,
\end{split}
\eeq
where we have defined the Metropolis ratio between two states $(z,y) \in \zset \times \bset$ as $\alpha_\theta(z, y) = \textrm{min}(1, \frac{\pi_\theta(z)  \prop{\theta}(z,y)}{\prop{\theta}(y,z) \pi_\theta(y)  })$.
Under H\ref{ass:bounded} and due to the fact that the threshold $\thresh$ leads to a symmetric positive definite covariance matrix with bounded non zero eigenvalues, then the following holds:
\beq\label{eq:boundgauss}
a n_{\sigma_1}(z - y) \leq \prop{\theta}(z,y)  \leq b n_{\sigma_2}(z - y)  \eqsp,
\eeq
for all $\theta \in \Theta$ and where $\sigma_1$ and $\sigma_2$ are the corresponding standard deviations of the two Gaussian distributions $n_{\sigma_1}$ and $n_{\sigma_2}$. 
We denote by $\rho_\theta$ the ratio $\frac{\pi_\theta(z)  \prop{\theta}(z,y)}{\prop{\theta}(y,z) \pi_\theta(y)  }$ and define the quantity 
\beq\label{eq:delta_main}
\delta = \textrm{inf}(\rho_\theta(z,y), \theta \in \Theta,  z \in \mathcal{O} ) > 0 \eqsp,
\eeq
where we have used assumptions H\ref{ass:bounded} and H\ref{ass:contlogpi}.
Then,
\beq\notag
\begin{split}
\Pi_\theta(z, \bset) &\geq  
\int_{\bset \cap \xset} \alpha_\theta(z, y) \prop{\theta}(z,y) \textrm{d}y  \geq \textrm{min}(1, \delta) m \int_{\bset} \mathsf{1}_\xset(z)  \textrm{d}y \eqsp.
\end{split}
\eeq
According to~\eqref{eq:delta_main}, we can find a compact set $\mathcal{O}$ such that $\Pi_\theta(z, \bset) \geq  \epsilon$ where $\epsilon = \textrm{min}(1, \delta) m \textbf{Z}$ where $\textbf{Z}$ is the normalizing constant of the pdf $\frac{1}{\textbf{Z}}\mathbf{1}_\xset(z)  \textrm{d}y$ and the proposal distribution is bounded from below by some quantity noted $m$.
The calculations above prove~\eqref{thm:main1}, \ie\ the existence of a small set for our family of transition kernels $(\Pi_\theta)_{\theta \in \Theta}$.

\medskip
\noindent \textbf{(ii) Drift condition and ergodicity: }
We begin by proving that $(\Pi_\theta)_{\theta \in \Theta}$ satisfies a drift property.
For a given EBM parameter $\theta \in \Theta$, we can see in~\citet{jarner2000geometric} that the drift condition boils down to proving that
\beq\notag
\sup \limits_{z \in \zset}  \frac{\Pi_\theta V_\theta(z)}{V_\theta(z)} < \infty \quad \textrm{and} \quad \lim \sup \limits_{|z| \to \infty}  \frac{\Pi_\theta V_\theta(z)}{V_\theta(z)} < 1 \eqsp,
\eeq
where $V_\theta$ is the \emph{drift function} defined in~\eqref{eq:driftfunction}
Let denote the acceptation set, \ie\ $\rho_\theta \geq 1$ by 
\beq\label{eq:accept}
\accept(z) \eqdef \{ y \in \zset, \rho_\theta(z,y) \geq 1 \}
\eeq
for any state $y \in \bset$ and its complementary set $\compaccept(z)$.
The remaining of the proof is composed of three main steps. 
\textcolor{magenta}{\textsc{Step (1)}} shows that for any $\theta \in \Theta$,
\begin{align}\notag
\lim \sup \limits_{|z| \to \infty}  \frac{\Pi_\theta V_\theta(z)}{V_\theta(z)} \leq 1 - \lim \inf \limits_{|z| \to \infty}  \int_{\accept(z)} \prop{\theta}(z,y)  \textrm{d}y \eqsp.
\end{align}
where the smoothness of a Gaussian pdf, assumption H\ref{ass:bounded} and a collection of inequalities based on \eqref{eq:accept}, its complentary set and the interval in \eqref{eq:boundgauss}.
Then, using an important intermediary result, stated in Lemma~\ref{lem:cone}, that initiates a relation between the set of accepted states noted $\accept(z)$ and the cone $\mathcal{P}(z)$ designed so that it does not depend on the model parameter $\theta$. 

\vspace{0.2in}

\begin{lemm}\label{lem:cone}
Define $\mathcal{P}(z) \eqdef \{ z- \ell \frac{z}{|z|} - \kappa \nu , \, \textrm{with} \quad \kappa < a - \ell  , \nu \in \{ \nu \in \rset^d, \| \nu \| < 1\}, |\nu - \frac{z- \ell \frac{z}{|z|} }{|z- \ell \frac{z}{|z|} |} \leq \frac{\epsilon}{2}   \}$ and $\accept(z) \eqdef \{ y \in \zset, \rho_\theta(z,y) \geq 1 \}$. Then for $z \in \zset$, $\mathcal{P}(z) \subset \accept(z)$.
\end{lemm}

\vspace{0.2in}

\noindent Noting the limit inferior as $\varliminf$, 
\textcolor{magenta}{\textsc{Step (2)}} establishes that $1 - \varliminf \limits_{|z| \to \infty}  \int_{\accept(z)} \prop{\theta}(z,y)  \textrm{d}y \leq 1 - c$ where $c$ is a constant, \emph{independent of all the other quantities} towards showing uniformity of the final result.
Finally, 
\textcolor{magenta}{\textsc{Step (3)}} uses the inequality $\Pi_\theta V_\theta(z) \leq \bar{\mu} V_\theta(z) + \bar{\delta} \mathsf{1}_{\mathcal{O}}(z)$ dependent of $\theta$ and defines the V function, independent of $\theta$, as $V(z) \eqdef V_1(z)^\alpha V_2(z)^{2\alpha}$ in order to establish the main result of Theorem~\ref{thm:thm1}, \ie
\begin{align}\notag
\Pi_\theta V(z) \leq \left(\frac{\bar{\mu}}{2 \epsilon^2} + \frac{\epsilon^2}{1 + \bar{\mu}} \right) V(z) +\frac{\bar{\delta}}{2 \epsilon^2} \mathsf{1}_{\mathcal{O}}(z) \eqsp.
\end{align}
Setting $\epsilon \eqdef \sqrt{\frac{\bar{\mu}(1+\bar{\mu})}{2}}$, $ \mu  \eqdef  \sqrt{\frac{2\bar{\mu}}{1+\bar{\mu}}}$ and $\delta \eqdef \frac{\bar{\delta}}{2 \epsilon^2}$ proves the uniformity of the inequality~\eqref{thm:main2}.

The complete proof is deferred in the appendix and is also developed in~\citet{allassonniere2015convergent} in the context of Bayesian Mixed Effect  models trained with the EM algorithm.


\subsection{Proof of Theorem~\ref{thm:thm1}}
\begin{Theorem*}
Assume H\ref{ass:bounded}-H\ref{ass:V2}.
For any $\theta \in \Theta$, there exists a drift function $V_\theta$, a set $\mathcal{O} \subset \zset$, a constant $0 < \epsilon \leq 1$ such that 
\beq
\Pi_\theta(z, \bset) \geq  \epsilon \int_{\bset} \mathsf{1}_\xset(z)  \textrm{d}y \eqsp.
\eeq
Moreover there exists $0 < \mu < 1$, $\delta > 0$ and a drift function $V$, now independent of $\theta$ such that for all $z \zset$:
\beq
\Pi_\theta V(z) \leq \mu V(z) + \delta \mathsf{1}_{\mathcal{O}}(z) \eqsp.
\eeq
\end{Theorem*}

\begin{proof}

Notations used throughout the proof are listed in Table~\ref{tab:notations}

\begin{table}[htbp]
\caption{Notations}
\begin{tabular}{r c p{11cm} }
\toprule
$\Pi_\theta$ & $\triangleq$ &  Transition kernel of the MCMC defined by~\eqref{eq:anila}\\
$\mathcal{O}$ & $\triangleq$ & Subset of $\rset^p$ and small set for kernel $\Pi_\theta$\\
$B(z,a)$  & $\triangleq$ & Ball around $z \in \zset$ of radius $a >0$\\
$\accept(z)$ & $\triangleq$ & Acceptance set at state $ \in \zset$ such that $\rho_\theta \geq 1$ \\
$\compaccept(z)$ & $\triangleq$ & Complementary set of  $\accept(z)$\\
$\prop{\theta}(z',z)$ & $\triangleq$ &  Probability density function of the Gaussian proposal\\
$\pi_{\theta}(\cdot)$ & $\triangleq$ &  Stationary/Target distribution under model $\theta \in \Theta$\\
$\Pi_\theta(z, z')$ & $\triangleq$ & Transition kernel from state $z$ to state $z'$\\
$n_{\sigma}(z)$ & $\triangleq$ & Pdf of a centered Normal distribution of standard deviation $\sigma >0$ \\
\bottomrule
\end{tabular}
\label{tab:notations}
\end{table}

The proof of our results are divided into two parts.
We first prove the existence of a set noted $\mathcal{O}$ as a small set for our family of transition kernels $(\Pi_\theta)_{\theta \in \Theta}$.
Proving a small set is crucial in order to show that for any state, the Markov Chain does not stay in the same state, and thus help in proving its irreducibility and aperiodicity.

Then, we will prove the drift condition towards a small set.
This condition is crucial to prove the convergence of the chain since it states that the kernels tend to attract elements into that set. 
finally, uniform ergodicity is established as a consequence of those drift conditions.

\medskip
\noindent \textbf{(i) Existence of small set: }
Let $\mathcal{O}$ be a compact subset of the state space $\zset$.
We also denote the probability density function (pdf) of the Gaussian proposal of Line 3 as $z \to \prop{\theta}(z',z)$ for any current state of the chain $z' \in \zset$ and dependent on the EBM model parameter $\theta$.
Given \algo's MCMC update, at iteration $t$, the proposal is a Gaussian distribution of mean $z_{t-1}^m+ \stepsize_t/2  \nabla f_{\theta_t}(z_{t-1}^m)$ and covariance $\sqrt{\stepsize_t} \mathsf{B}_t$.

We recall the definition of the transition kernel in the case of a Metropolis adjustment and for any model parameter $\theta \in \Theta$ and state $z \in \zset$:

\beq 
\Pi_\theta(z, \bset) = \int_{\bset} \alpha_\theta(z, y) \prop{\theta}(z,y) \textrm{d}y + \mathsf{1}_{\bset(z)}\int_{\zset} (1 - \alpha_\theta(z, y)) \prop{\theta}(z,y) \textrm{d}y \eqsp,
\eeq

where we have defined the Metropolis ratio between two states $z \in \zset$ and $y \in \bset$ as $\alpha_\theta(z, y) = \textrm{min}(1, \frac{\pi_\theta(z)  \prop{\theta}(z,y)}{\prop{\theta}(y,z) \pi_\theta(y)  })$.
Thanks to Assumption H\ref{ass:bounded} and due to the fact that the threshold $\thresh$ leads to a symmetric positive definite covariance matrix with bounded non zero eigenvalues implies that the proposal distribution can be bounded by two zero-mean Gaussian distributions as follows:
\beq\label{eq:twogauss}
a n_{\sigma_1}(z - y) \leq \prop{\theta}(z,y)  \leq b n_{\sigma_2}(z - y) \quad \textrm{for all} \quad \theta \in \Theta\eqsp,
\eeq
where $\sigma_1$ and $\sigma_2$ are the corresponding standard deviation of the distributions and $a$ and $b$ are some scaling factors.

We denote by $\rho_\theta$ the ratio $\frac{\pi_\theta(z)  \prop{\theta}(z,y)}{\prop{\theta}(y,z) \pi_\theta(y)  }$ and given the assumptions H\ref{ass:bounded} and H\ref{ass:contlogpi}, define the quantity 
\beq\label{eq:delta}
\delta = \textrm{inf}(\rho_\theta(z,y), \theta \in \Theta, \quad z \in \mathcal{O} ) > 0 \eqsp.
\eeq
Likewise, the proposal distribution is bounded from below by some quantity noted $m$.
Then,
\beq
\Pi_\theta(z, \bset) \geq  \int_{\bset \cap \xset} \alpha_\theta(z, y) \prop{\theta}(z,y) \textrm{d}y \geq \textrm{min}(1, \delta) m \int_{\bset} \mathsf{1}_\xset(z)  \textrm{d}y\eqsp.
\eeq

Then, given the definition of~\eqref{eq:delta}, we can find a compact set $\mathcal{O}$ such that $\Pi_\theta(z, \bset) \geq  \epsilon$ where $\epsilon = \textrm{min}(1, \delta) m \textbf{Z}$ where $\textbf{Z}$ is the normalizing constant of the pdf $\frac{1}{\textbf{Z}}\mathbf{1}_\xset(z)  \textrm{d}y$.
The calculations above prove~\eqref{thm:main1}, \ie\ the existence of a small set for our family of transition kernels $(\Pi_\theta)_{\theta \in \Theta}$.

\newpage

\noindent \textbf{(ii) Drift condition and ergodicity: }
We first need to prove the fact that our family of transition kernels $(\Pi_\theta)_{\theta \in \Theta}$ satisfies a drift property.

For a given EBM model parameter $\theta \in \Theta$, we can see in~\citet{jarner2000geometric} that the drift condition boils down to proving that for the drift function noted $V_\theta$ and defined in~\eqref{eq:driftfunction}, we have
\beq\label{mainproof}
\sup \limits_{z \in \zset}  \frac{\Pi_\theta V_\theta(z)}{V_\theta(z)} < \infty \quad \textrm{and} \quad \lim \sup \limits_{|z| \to \infty}  \frac{\Pi_\theta V_\theta(z)}{V_\theta(z)} < 1 \eqsp.
\eeq

Throughout the proof, the model parameter is set to an arbitrary $\theta \in \Theta$.
Let denote the acceptation set, \ie\ $\rho_\theta \geq 1$ by $\accept(z) \eqdef \{ y \in \zset, \rho_\theta(z,y) \geq 1 \}$ for any state $y \in \bset$ and its complementary set $\compaccept(z)$.

\medskip
\noindent \textsc{Step (1): } Following our definition of the drift function in~\eqref{eq:driftfunction} we obtain:

\begin{align}\label{eq:main1}
 \frac{\Pi_\theta V_\theta(z)}{V_\theta(z)} & = \int_{\accept(z)}  \prop{\theta}(z,y) \frac{V_\theta(y)}{V_\theta(z)} \textrm{d}y +  \int_{\compaccept(z)} \frac{\pi_\theta(y)\prop{\theta}(y,z)}{\pi_\theta(z)\prop{\theta}(z,y)} \prop{\theta}(z,y) \frac{V_\theta(y)}{V_\theta(z)} \textrm{d}y \\ \nonumber &+  \int_{\compaccept(z)} (1 - \frac{\pi_\theta(y)\prop{\theta}(y,z)}{\pi_\theta(z)\prop{\theta}(z,y)}) \prop{\theta}(z,y)  \textrm{d}y\\
 &  \overset{(a)}{\leq} \int_{\accept(z)}  \prop{\theta}(z,y) \frac{\pi_\theta(y)^{-\beta}}{\pi_\theta(z)^{-\beta}} \textrm{d}y  + \int_{\compaccept(z)} \prop{\theta}(z,y) \frac{\pi_\theta(y)^{1-\beta}}{\pi_\theta(z)^{1-\beta}}\textrm{d}y +  \int_{\compaccept(z)} \prop{\theta}(z,y)  \textrm{d}y \eqsp,
\end{align}
where (a) is due to~\eqref{eq:driftfunction}.

Furthermore, according to~\eqref{eq:twogauss}, we thus have that, for any state $z$ in the acceptance set $\accept(z)$:
\beq \label{eq:comp}
\int_{\accept(z)}  \prop{\theta}(z,y) \frac{\pi_\theta(y)^{-\beta}}{\pi_\theta(z)^{-\beta}} \textrm{d}y  \leq  b \int_{\accept(z)}  n_{\sigma_2}(y-z)  \textrm{d}y \eqsp.
\eeq
For any state $z$ in the complementary set of the acceptance set, noted $\compaccept(z)$, we also have the following:
\beq
\int_{\compaccept(z)} \prop{\theta}(z,y) \frac{\pi_\theta(y)^{1-\beta}}{\pi_\theta(z)^{1-\beta}}\textrm{d}y \leq \int_{\compaccept(z)} \prop{\theta}(z,y)^{1- \beta} \prop{\theta}(y,z)^{\beta}  \textrm{d}y \leq b \int_{\compaccept(z)} n_{\sigma_2}(z - y)  \textrm{d}y \eqsp.
\eeq
While we can define the level set of the stationary distribution $\pi_\theta$ as $\mathcal{L}_{\pi_\theta(y)} = \{ z \in \zset, \pi_\theta(z) = \pi_\theta(y) \}$ for some state $y \in \bset$, a neighborhood of that level set is defined as $\mathcal{L}_{\pi_\theta(y)}(p) = \{z \in  \mathcal{L}_{\pi_\theta(y)}, z + t \frac{z}{|z|}, |t| \leq p \}$.
H\ref{ass:bounded} ensures the existence of a radial $r$ such that for all $z \in \zset, |z| \geq r$, then $0 \in \mathcal{L}_{\pi_\theta(y)}$ with $\pi_\theta(z) >  \pi_\theta(y)$.
Since the function $y \to n_{\sigma_2}(y - z)$ is smooth, it is known that there exists a constant $a >0$ such that for $\epsilon >0$, we have that 
\beq\label{eq:lowandup}
\int_{B(z,a)}  n_{\sigma_2}(y - z) \textrm{d}y \geq 1 - \epsilon \quad \textrm{and} \quad \int_{B(z,a) \cap \mathcal{L}_{\pi_\theta(y)}(p) }  n_{\sigma_2}(y - z) \textrm{d}y \leq  \epsilon \eqsp,
\eeq
for some $p$ small enough and where $B(z,a)$ denotes the ball around $z \in \zset$ of radius $a$.
Then combining~\eqref{eq:comp} and~\eqref{eq:lowandup} we have that:
\beq
\int_{\accept(z) \cap B(z,a) \cap \mathcal{L}_{\pi_\theta(y)}(p) }  \prop{\theta}(z,y) \frac{\pi_\theta(y)^{-\beta}}{\pi_\theta(z)^{-\beta}} \textrm{d}y  \leq  b \epsilon \eqsp.
\eeq
Conversely, we can define the following set $\mathcal{A} = \accept(z) \cap B(z,a) \cap \mathcal{L}^+$ where $u \in \mathcal{L}^+$ if $u \in \mathcal{L}_{\pi_\theta(y)}(p)$ and $\phi_\theta(u) > \pi_\theta(p)$.
Then using the second part of H\ref{ass:bounded}, there exists a radius $r' > r + a$, such that for $z \in \zset$ with $|z| \geq r'$ we have
\beq
\int_{\mathcal{A}} (\frac{\pi_\theta(y)}{\pi_\theta(z)})^{1-\beta} \prop{\theta}(y,z) \textrm{d}y \leq \mathsf{d}(p, r')^{1-\beta}  b \int_{\accept(z)}  n_{\sigma_2}(y-z)  \textrm{d}y\leq b \mathsf{d}(p, r')^{1-\beta} \eqsp,
\eeq
where $\mathsf{d}(p, r') = \sup \limits_{|z| > r'} \frac{\pi_\theta(z + p \frac{z}{|z|})}{\pi_\theta(z)}$. 
Note that H\ref{ass:bounded} implies that $\mathsf{d}(p, r') \to 0$ when $r' \to \infty$.
Likewise with  $\mathcal{A} = \accept(z) \cap B(z,a) \cap \mathcal{L}^-$ we have
\beq
\int_{\mathcal{A}} (\frac{\pi_\theta(y)}{\pi_\theta(z)})^{-\beta} \prop{\theta}(z,y) \textrm{d}y  \leq b \mathsf{d}(p, r')^{\beta} \eqsp.
\eeq
Same arguments can be obtained for the second term of~\eqref{eq:main1}, \ie\ $\prop{\theta}(z,y) \frac{\pi_\theta(y)^{1-\beta}}{\pi_\theta(z)^{1-\beta}}$ and we obtain, plugging the above in~\eqref{eq:main1} that:
\begin{align}
\lim \sup \limits_{|z| \to \infty}  \frac{\Pi_\theta V_\theta(z)}{V_\theta(z)} \leq \lim \sup \limits_{|z| \to \infty}  \int_{\compaccept(z)} \prop{\theta}(z,y)  \textrm{d}y \eqsp.
\end{align}
Since $\compaccept(z)$ is the complementary set of $\accept(z)$, the above inequality yields
\begin{align}
\lim \sup \limits_{|z| \to \infty}  \frac{\Pi_\theta V_\theta(z)}{V_\theta(z)} \leq 1 - \lim \inf \limits_{|z| \to \infty}  \int_{\accept(z)} \prop{\theta}(z,y)  \textrm{d}y \eqsp.
\end{align}

\medskip
\noindent \textsc{Step (2): } The final step of our proof consists in proving that $1 - \lim \inf \limits_{|z| \to \infty}  \int_{\accept(z)} \prop{\theta}(z,y)  \textrm{d}y \leq 1 - c$ where $c$ is a constant, independent of all the other quantities.

Given that the proposal distribution is a Gaussian and using assumption H\ref{ass:bounded} we have the existence of a constant $c_a$ depending on $a$ as defined above (the radius of the ball $B(z,a)$) such that
\beq\label{eq:gauss}
\frac{\pi_\theta(z)}{\pi_\theta(z- \ell \frac{z}{|z|})} \leq  c_a \leq \inf \limits_{y \in B(z,a)} \frac{\prop{\theta}(y, z)}{\prop{\theta}(z, y)} \quad \textrm{for any} \, z \in \zset, |z| \geq r^* \eqsp.
\eeq

Then for any $|z| \geq r^*$, we obtain that $z- \ell \frac{z}{|z|} \in \accept(z)$.
A particular subset of $\accept(z)$ used throughout the rest of the proof is the cone defined as 
\beq\label{eq:defcone}
\mathcal{P}(z) \eqdef \{ z- \ell \frac{z}{|z|} - \kappa \nu , \, \textrm{with} \quad i < a - \ell  , \nu \in \{ \nu \in \rset^d, \| \nu \| < 1\}, |\nu - \frac{z- \ell \frac{z}{|z|} }{|z- \ell \frac{z}{|z|} |} \leq \frac{\epsilon}{2}   \} \eqsp.
\eeq

Using Lemma~\ref{lem:cone}, we have that $\mathcal{P}(z) \subset \accept(z)$.
Then, we observe that  
\beq 
 \int_{\accept(z)} \prop{\theta}(z,y)  \textrm{d}y \overset{(a)}{\geq}  \int_{\accept(z)}a n_{\sigma_1}(y- z)  \textrm{d}y \overset{(b)}{\geq} a \int_{\mathcal{P}(z)}  n_{\sigma_1}(y-z)  \textrm{d}y \eqsp,
 \eeq
where we have used~\eqref{eq:twogauss} in (a) and applied Lemma~\ref{lem:cone} in (b).

If we define the translation of vector $z \in \zset$ by the operator $\mathcal{I} \subset \rset^d \to T_z(\mathcal{I})$, then
\beq\label{eq:constant}
 \int_{\accept(z)} \prop{\theta}(z,y)  \textrm{d}y \geq a \int_{\mathcal{P}(z)}  n_{\sigma_1}(y-z)  \textrm{d}y =  \int_{T_z(\mathcal{P}(z))}  n_{\sigma_1}(y-z)  \textrm{d}y \eqsp.
\eeq

Recalling the objective of \noindent \textsc{Step (2)} that is to find a constant $c$ such that $1 - \lim \inf \limits_{|z| \to \infty}  \int_{\accept(z)} \prop{\theta}(z,y)  \textrm{d}y \leq 1 - c$, we deduce from~\eqref{eq:constant} that since the set $\mathcal{P}(z)$ does not depend on the EBM model parameter $\theta$ and that once translated by $z$ the resulting set $T_z(\mathcal{P}(z))$ is independent of $z$ (but depends on $\ell$, see definition~\eqref{eq:defcone}, then the integral $ \int_{T_z(\mathcal{P}(z))}  n_{\sigma_1}(y-z)  \textrm{d}y$ in~\eqref{eq:constant} is independent of $z$ thus concluding on the existence of the constant $c$ such that 
$$\lim \sup \limits_{|z| \to \infty}  \frac{\Pi_\theta V_\theta(z)}{V_\theta(z)} \leq 1- c \eqsp.$$ 
Thus proving the second part of~\eqref{mainproof} which is the main drift condition we ought to demonstrate.
The first part of~\eqref{mainproof} can be proved by observing that $  \frac{\Pi_\theta V_\theta(z)}{V_\theta(z)} $ is smooth on $\zset$ according to H\ref{ass:contlogpi} and by construction of the transition kernel. Smoothness implies boundedness on the compact $\zset$.

\medskip
\noindent \textsc{Step (3): } 
We now use the main proven equations in~\eqref{mainproof} to derive the second result~\eqref{thm:main2} of Theorem~\ref{thm:thm1}.

We will begin by showing a similar inequality for the drift function $V_\theta$, thus not having uniformity, as an intermediary step.
The Drift property is a consequence of \textsc{Step (2)} and~\eqref{eq:constant} shown above.
Thus, there exists $0 < \bar{\mu} < 1$, $\bar{\delta} > 0$ such that for all $z \zset$:
\beq\label{eq:driftvtheta}
\Pi_\theta V_\theta(z) \leq \bar{\mu} V_\theta(z) + \bar{\delta} \mathsf{1}_{\mathcal{O}}(z) \eqsp,
\eeq
where $V_\theta$ is defined by~\eqref{eq:driftfunction}.
Using the two functions defined in~\eqref{eq:vfunctions}, we define for $z \in \zset$, the $V$ function independent of $\theta$ as follows:
\beq\label{eq:defv}
V(z) = V_1(z)^\alpha V_2(z)^{2\alpha} \eqsp,
\eeq
where $0 < \alpha < \textrm{min}(\frac{1}{2\beta},\frac{a_0}{3})$, $a_0$ is defined in H\ref{ass:V2} and $\beta$ is defined in~\eqref{eq:driftfunction}.
Thus for $\theta \in \Theta$, $z \in \zset$ and $\epsilon >0$:
\begin{align}\notag
\Pi_\theta V(z) & = \int_{\zset} \Pi_\theta(z,y) V_1(y)^\alpha V_2(y)^{2\alpha} \textrm{d}y\\ \notag \eqsp,
& \overset{(a)}{\leq} \frac{1}{2} \int_{\zset} \Pi_\theta(z,y) (\frac{1}{\epsilon^2}V_1(y)^{2\alpha} + \epsilon^2 V_2(y)^{4\alpha}) \textrm{d}y \eqsp,\\ 
& \overset{(b)}{\leq} \frac{1}{2\epsilon^2} \int_{\zset} \Pi_\theta(z,y) V_\theta(y)^{2\alpha} + \frac{\epsilon^2}{2}  \int_{\zset} \Pi_\theta(z,y)  V_2(y)^{4\alpha} \textrm{d}y \label{eq:uniform1} \eqsp,
\end{align}
where we have used the Young's inequality in (a) and the definition of $V_1$, see~\eqref{eq:vfunctions}, in (b).
Then plugging~\eqref{eq:driftvtheta} in~\eqref{eq:uniform1}, we have
\begin{align}
\Pi_\theta V(z) & \leq \frac{1}{2\epsilon^2} (\bar{\mu} V_\theta(z)^{2\alpha} + \bar{\delta} \mathsf{1}_{\mathcal{O}}(z) ) + \frac{\epsilon^2}{2}  \int_{\zset} \Pi_\theta(z,y)  V_2(y)^{4\alpha} \textrm{d}y \eqsp,\\
& \leq \frac{\bar{\mu}}{2 \epsilon^2} V(z) +\frac{\bar{\delta}}{2 \epsilon^2} \mathsf{1}_{\mathcal{O}}(z)  + \frac{\epsilon^2}{2}  \int_{\zset} \Pi_\theta(z,y)  V_2(y)^{4\alpha} \textrm{d}y \eqsp,\\
& \leq \frac{\bar{\mu}}{2 \epsilon^2} V(z) +\frac{\bar{\delta}}{2 \epsilon^2} \mathsf{1}_{\mathcal{O}}(z)  + \frac{\epsilon^2}{2} \sup \limits_{\theta \in \Theta, z \in \zset} \int_{\zset} \Pi_\theta(z,y)  V_2(y)^{4\alpha} \textrm{d}y \eqsp,\\
& \leq \frac{\bar{\mu}}{2 \epsilon^2} V(z) +\frac{\bar{\delta}}{2 \epsilon^2} \mathsf{1}_{\mathcal{O}}(z)  + \frac{\epsilon^2}{1 + \bar{\mu}}V(z) \eqsp,\\
& \leq \left(\frac{\bar{\mu}}{2 \epsilon^2} + \frac{\epsilon^2}{1 + \bar{\mu}} \right) V(z) +\frac{\bar{\delta}}{2 \epsilon^2} \mathsf{1}_{\mathcal{O}}(z) \eqsp,
\end{align}

where we have used~\eqref{eq:defv} and the assumption H\ref{ass:V2} in the last inequality, ensuring the existence of such exponent $\alpha$.

Setting $\epsilon \eqdef \sqrt{\frac{\bar{\mu}(1+\bar{\mu})}{2}}$, $ \mu  \eqdef  \sqrt{\frac{2\bar{\mu}}{1+\bar{\mu}}}$ and $\delta \eqdef \frac{\bar{\delta}}{2 \epsilon^2}$ proves the uniform ergodicity in~\eqref{thm:main2} and concludes the proof of Theorem~\ref{thm:thm1}.
\end{proof}

\newpage
\subsection{Proof of Lemma~\ref{lem:cone}}

\begin{Lemma*}
Define $\mathcal{P}(z) \eqdef \{ z- \ell \frac{z}{|z|} - \kappa \nu , \, \textrm{with} \quad \kappa < a - \ell  , \nu \in \{ \nu \in \rset^d, \| \nu \| < 1\}, |\nu - \frac{z- \ell \frac{z}{|z|} }{|z- \ell \frac{z}{|z|} |} \leq \frac{\epsilon}{2}   \}$ and $\accept(z) \eqdef \{ y \in \zset, \rho_\theta(z,y) \geq 1 \}$. Then for $z \in \zset$, $\mathcal{P}(z) \subset \accept(z)$.
\end{Lemma*}

\begin{proof}

In order to show the inclusion of the set $\mathcal{P}(z)$ in $\accept(z)$ we start by selecting the quantity $y = z- \ell \frac{z}{|z|} - \kappa \nu$ for $z\in \zset$ and $\kappa < a - \ell $ where $a$ is the radius of the ball used in~\eqref{eq:lowandup} such that $y \in \mathcal{P}(z)$.
We will now show that $y \in \accept(z)$.

By the generalization of Rolle's theorem applied on the stationary distribution $\pi_\theta$, we guarantee the existence of some $\kappa^*$ such that:
\begin{align}
\nabla \pi_\theta( z- \ell \frac{z}{|z|} - \kappa^* \nu)  = \frac{\pi_\theta(y) - \pi_\theta(z- \ell \frac{z}{|z|})}{y - (z- \ell \frac{z}{|z|})} = - \frac{\pi_\theta(y) - \pi_\theta(z- \ell \frac{z}{|z|})}{\kappa \nu} \eqsp.
\end{align}

Expanding $\nabla \pi_\theta( z- \ell \frac{z}{|z|} - \kappa^* \nu)$ yields:
\beq\label{eq:interlem}
\pi_\theta(y) - \pi_\theta(z- \ell \frac{z}{|z|}) = - \kappa \nu \frac{z- \ell \frac{z}{|z|} - \kappa^* \nu}{|z- \ell \frac{z}{|z|} - \kappa^* \nu|} |\nabla \pi_\theta( z- \ell \frac{z}{|z|} - \kappa^* \nu)| \eqsp.
\eeq

Yet, under assumption H\ref{ass:bounded}, there exists $\epsilon$ such that 
$$
 \frac{\nabla f_{\theta}(z) }{|\nabla f_{\theta}(z) |}  \frac{z }{|z|} \leq -\epsilon \eqsp,
 $$
 and for any $y \in \mathcal{P}(z)$ we note that $\frac{y }{|y|} - \frac{z }{|z|}|\leq \frac{\epsilon}{2}$, by construction of the set.
 Thus, 
\beq\label{eq:finallem}
  \frac{\nabla f_{\theta}(y) }{|\nabla f_{\theta}(y) |}  \nu  =  \frac{\nabla f_{\theta}(y) }{|\nabla f_{\theta}(y) |} (\nu - \frac{z- \ell \frac{z}{|z|} }{|z- \ell \frac{z}{|z|} |}) +  \frac{\nabla f_{\theta}(y) }{|\nabla f_{\theta}(y) |} (\nu - \frac{z- \ell \frac{z}{|z|} }{|z- \ell \frac{z}{|z|} |} - \frac{y }{|y|} ) + \frac{\nabla f_{\theta}(y) }{|\nabla f_{\theta}(y) |}  \frac{y }{|y |} \leq 0 \eqsp,
\eeq
 where $\nu$ is used in the definition of $\mathcal{P}(z)$.
Additionally we let $  \frac{\nabla f_{\theta}(y) }{|\nabla f_{\theta}(y) |}  \nu  $ denote the vector multiplication between the normalized gradient and $\nu$.
Then plugging~\eqref{eq:finallem} into~\eqref{eq:interlem} leads to $\pi_\theta(y) - \pi_\theta(z- \ell \frac{z}{|z|}) \geq 0$ and $y \in \mathcal{P}(z)$ implies, using~\eqref{eq:gauss}, that $\pi_\theta(y) \geq \pi_\theta(z- \ell \frac{z}{|z|}) \geq \frac{1}{c_a} \pi_\theta(z)$. 
 Finally $y \in \mathcal{P}(z)$ implies that $y \in \accept(z)$, concluding the proof of Lemma~\ref{lem:cone}.
 
\end{proof}

\end{document}